\documentclass[]{jingdong}
\usepackage[toc,page,header]{appendix}
\usepackage{latexsym}
\usepackage[T1]{fontenc}
\usepackage[utf8]{inputenc}
\usepackage{microtype}
\usepackage{inconsolata}
\usepackage{graphicx}
\usepackage{url}            
\usepackage{booktabs}       
\usepackage{amsfonts}       
\usepackage{nicefrac}       
\usepackage{stackengine}
\usepackage{microtype}      
\usepackage{colortbl}
\usepackage{xcolor}
\usepackage{amsmath}
\usepackage{amssymb}
\usepackage{amsthm}
\usepackage{mathrsfs}
\usepackage{pifont}
\usepackage{MnSymbol}
\usepackage{balance}
\usepackage{enumitem}
\usepackage{listings}
\usepackage{xcolor}
\usepackage{natbib}
\usepackage{multicol}

\AtBeginDocument{%
  \providecommand\BibTeX{{%
    \normalfont B\kern-0.5em{\scshape i\kern-0.25em b}\kern-0.8em\TeX}}}

\makeatletter
\DeclareRobustCommand\onedot{\futurelet\@let@token\@onedot}
\def\@onedot{\ifx\@let@token.\else.\null\fi}

\usepackage{setspace}
\usepackage{mathtools}

\usepackage{multirow,booktabs}
\usepackage{subcaption}

\newcommand{\owo}[1]{\textsc{OAgents}}

\definecolor{lightgreen}{RGB}{144, 238, 144} 
\definecolor{lightred}{RGB}{255, 105, 97}

\newtcolorbox{promptbox}[2][Prompt]{
colback=black!5!white,
arc=5pt, 
boxrule=0.5pt,
fonttitle=\bfseries,
title=#1, 
before upper={\small}, fontupper=\fontfamily{ptm}\selectfont,
colframe=#2, 
}
\definecolor{ogreen}{RGB}{34, 139, 34}

\usepackage{cleveref}
\AtBeginDocument{%
}
\theoremstyle{plain}
\newtheorem{theorem}{Theorem}[section]

\theoremstyle{definition}

\theoremstyle{remark}
\usepackage{subcaption}

\usepackage[textsize=tiny]{todonotes}
\usepackage{fontawesome}

\usepackage{xcolor}
\usepackage{graphicx}
\usepackage[edges]{forest}

\usepackage{amsmath}

\title{Leveraging Error Diversity in Group Rollouts for Reinforcement Learning}
\author[1,2,*]{Wenpu~Liu}
\author[1,2,*]{Yuqi~Xu}
\author[1,2,*]{Weichu~Xie}
\author[2]{Yongfu~Zhu}
\author[2,3]{Shuai~Dong}
\author[1,2]{Ziyue~Wang}
\author[2]{Wenqi~Shao}
\author[2]{Xiaoying~Zhang}
\author[1,\dagger]{Tong~Yang}
\author[2]{Nan~Duan}
\author[2]{Jiaqi~Wang}

\affiliation[1]{Peking University}
\affiliation[2]{JD.COM}
\affiliation[3]{Shanghai Innovation Institute}

\contribution[*]{Equal contribution}
\contribution[\dagger]{Corresponding author.}

\usepackage{amsmath}
\usepackage{amsthm}
\usepackage{amssymb}
\usepackage{graphicx}
\usepackage{subcaption}
\usepackage{placeins}
\usepackage{float}
\usepackage{algorithm}
\usepackage{algpseudocode}
\usepackage{multirow}
\usepackage{graphicx}
\usepackage{caption}
\usepackage[table,x11names]{xcolor}

\usepackage{booktabs}
\usepackage{array}
\usepackage{dashrule} 
\usepackage{tcolorbox}
\usepackage{multicol}
\tcbuselibrary{skins,breakable}

\definecolor{mygray}{HTML}{e6e6e6}
\definecolor{myblue}{HTML}{cce0f5}
\definecolor{mypurple}{HTML}{e9e0f2}
\definecolor{lightred}{HTML}{FEEBEB}
\definecolor{casered}{HTML}{da291c}
\newcommand{\lightrule}{\arrayrulecolor{gray!30}\midrule\arrayrulecolor{black}}
\abstract{
    Reinforcement Learning from Verifiable Rewards (RLVR) typically samples multiple responses per prompt and assigns binary rewards based on individual correctness, yet the collective structure of the group output---specifically, the distribution of errors---is largely discarded. We identify this as a missed opportunity: empirical analysis reveals that error diversity within a group is a strong predictor of training success, with problems eliciting diverse wrong answers benefiting substantially more from RLVR than those producing homogeneous failures. Motivated by this observation, we propose \textbf{Error Diversity Advantage Shaping (EDAS)}, a lightweight, algorithm-agnostic technique that modulates the advantage signal for incorrect rollouts based on intra-group error diversity. EDAS amplifies penalties for dominant, repeated errors and attenuates penalties for rare, exploratory ones, thereby encouraging the model to maintain diverse reasoning paths and discouraging error perseveration. Crucially, EDAS operates as a simple post-hoc adjustment that can be seamlessly integrated into any RLVR algorithm. We validate EDAS on top of several mainstream RLVR methods across a series of models and seven challenging math benchmarks, demonstrating consistent improvements. Notably, EDAS yields an average improvement of 6.29 points over DAPO on Qwen3-8B across seven benchmarks, confirming that exploiting the latent information in group rollouts is a broadly effective strategy for strengthening RLVR.
}
\date{May 16, 2026}
\checkdata[Code]{\url{https://github.com/EDAS-jd/EDAS}}
\correspondence{\email{yangtong@pku.edu.cn}}

\begin{document}
\maketitle

\section{Introduction}

The rapid advance of large language models (LLMs) has been propelled by increasingly sophisticated post-training strategies, from supervised instruction tuning to reinforcement learning from human feedback~\cite{PPO,jaech2024openai,liu2024deepseek}. More recently, Reinforcement Learning from Verifiable Rewards (RLVR) has emerged as a powerful paradigm for domains where solutions are unknown but correctness is programmatically checkable---such as formal mathematics and competitive programming~\cite{shao2024deepseekmath,guo2025deepseek}. In the RLVR pipeline, the model simultaneously acts as generator and learner: it samples candidate reasoning paths, receives binary correctness feedback from a deterministic verifier, and refines its policy accordingly, enabling continuous self-evolution beyond human-provided examples.

A defining feature of modern RLVR algorithms is group-based rollout: for each training prompt, the policy generates a group of responses, and the reward signal is constructed from the correctness labels of individual outputs within the group. However, current methods treat these labels in isolation, using only the binary correct/incorrect distinction per sample to compute advantages~\cite{shao2024deepseekmath,yu2025dapo}. The collective structure of the group output---in particular, the distribution of errors across the incorrect responses---is entirely discarded. We argue that this represents a significant, overlooked information source.

Our empirical analysis (\autoref{sec:motivation}) reveals that the \textbf{diversity of wrong answers within a rollout group} is a strong predictor of whether RLVR training will improve the model’s performance on that problem---problems with diverse errors exhibit substantially higher improvement rates, whereas those with homogeneous failures see limited gains. Motivated by this observation, we propose \textbf{Error Diversity Advantage Shaping (EDAS)}, a lightweight, algorithm-agnostic technique that reshapes the advantage signal for incorrect rollouts based on intra-group error diversity: it amplifies penalties for dominant, repeated errors and attenuates penalties for rare, exploratory ones, thereby encouraging broader exploration and discouraging error perseveration. Crucially, EDAS operates as a simple post-hoc advantage adjustment that can be applied on top of any group-relative RLVR method without modifying the core training algorithm.

\begin{figure}[htbp]
\centering
\small
\begin{tcolorbox}[
  colback=casered!6,
  colframe=casered!100,
  coltitle=white,
  fonttitle=\bfseries,
  title={AIME 2026 \#21 Find the sum of all $r$ such that the circle of radius $r$ centered at $(4, 39)$ is tangent to $2y = x^2 - 8x + 12$. \hfill Ground Truth: 50},
  boxrule=0.8pt,
  arc=2pt,
  left=6pt, right=6pt, top=4pt, bottom=4pt
]


\textbf{Shared setup.} All 32 samples define the squared-distance function $f(x)$ and find two critical values: $f{=}81$ (at $x{=}4{\pm}4\sqrt{5}$) and $f{=}1681$ (at $x{=}4$), corresponding to tangency radii $r{=}9$ and $r{=}41$. The correct answer requires recognizing \emph{both} as valid: $9+41=50$. Samples diverge at how they interpret these critical values:\\[6pt]

\textbf{Answer = 9} {\small\textcolor{gray}{(9/32)}} --- Keeps only $r{=}9$, discards $r{=}41$. Treats the global minimum as the \emph{only} tangency.\\[1pt]
\colorbox{casered!10}{\textit{``The \textbf{minimum value} of $f(x)$ is 81\,...\,So $r = \sqrt{81} = \boxed{9}$.’’}}\\[6pt]

\textbf{Answer = 41} {\small\textcolor{gray}{(4/32)}} --- Keeps only $r{=}41$, discards $r{=}9$. Misidentifies $f(4){=}1681$ as the global minimum.\\[1pt]
\colorbox{casered!10}{\textit{``The minimum value of $f(x)$ is $1681$\,...\,the only tangent radius is $r{=}\boxed{41}$.’’}}\\[6pt]

\textbf{Answer = 0} {\small\textcolor{gray}{(6/32)}} --- Identifies only \emph{one} radius (either 9 or 41) but allows $r<0$, so $+r$ and $-r$ cancel to 0.\\[1pt]
\colorbox{casered!10}{\textit{``Since $r^2{=}81$ implies $r{=}\pm 9$\,...\,sum $= 9 + (-9) = \boxed{0}$.’’}}

\end{tcolorbox}

\caption{Case study on Qwen3-8B (AIME 2026 \#21, 32 rollouts). The three dominant wrong answers (9, 41, and 0) each correspond to a distinct misinterpretation of the two critical values of $f(x)$, and trajectories converging on the same incorrect answer follow near-identical reasoning paths. This illustrates that the diversity of final wrong answers within a group serves as a reliable proxy for the diversity of underlying reasoning failures, motivating EDAS's use of answer-level partitioning.}
\label{fig:case_study}
\end{figure}

We validate EDAS on mathematical reasoning and code generation tasks, integrating it into both GRPO~\cite{shao2024deepseekmath} and DAPO~\cite{yu2025dapo} and comparing against several baselines. EDAS consistently yields significant improvements, with an average gain of 6.29 points over DAPO on Qwen3-8B across seven math benchmarks and 1.45 points on code generation. The key contributions of this work are as follows:

\begin{itemize}
   \item \textbf{Empirical Insight into Error Diversity:} We identify intra-group error distribution as a critical, yet previously underexplored, determinant of training efficacy in RLVR. Our analysis reveals that environments eliciting diverse reasoning failures yield substantially richer learning signals, whereas homogeneous error distributions severely bottleneck policy optimization.

   \item \textbf{Error Diversity Advantage Shaping:} We propose a lightweight and algorithm-agnostic framework that dynamically calibrates advantage signals based on information-theoretic error diversity. By amplifying penalties for dominant, entrenched failures and attenuating them for rare errors, EDAS discourages error perseveration and incentivizes the  exploration of alternative reasoning trajectories.

   \item \textbf{Comprehensive Empirical Validation:} Extensive evaluations demonstrate the robustness of EDAS across models of different scales and two foundational RL algorithms on math reasoning and code generation tasks. Notably, our method establishes consistent performance gains, yielding an average absolute improvement of 6.29 points on math reasoning and 1.45 points on code generation over the SOTA baseline.
\end{itemize}

\begin{figure}[H]
    \centering
    \begin{subfigure}{0.32\textwidth}
        \centering
        \includegraphics[width=\linewidth]{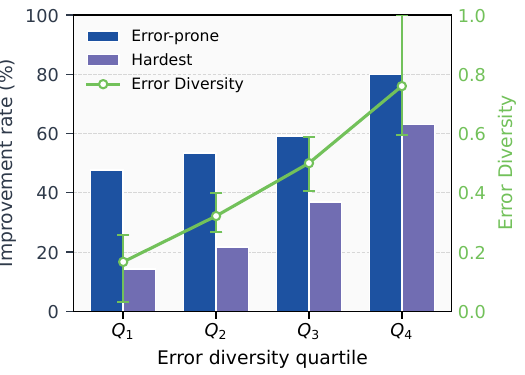}
        \caption{}
        \label{fig:uprate}
    \end{subfigure}
    \hfill
    \begin{subfigure}{0.32\textwidth}
         \centering
         \includegraphics[width=\linewidth]{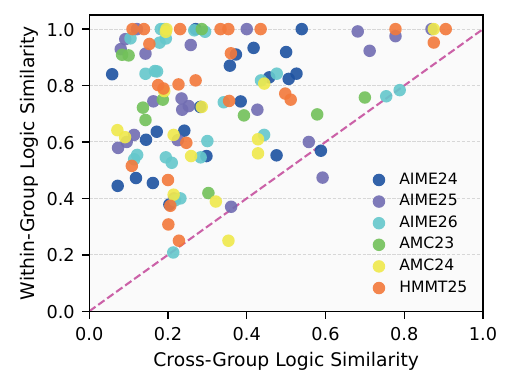}
         \caption{}
         \label{fig:logic_similarity}
     \end{subfigure}
         \hfill
    \begin{subfigure}{0.32\textwidth}
         \centering
         \includegraphics[width=\linewidth]{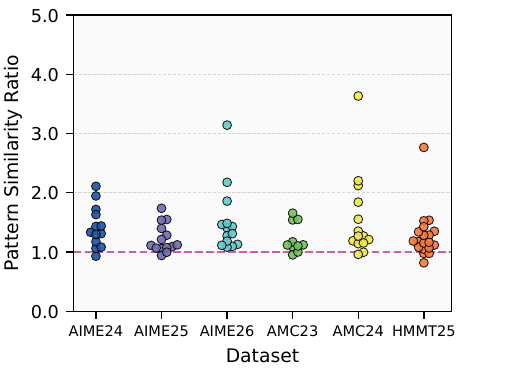}
         \caption{}
         \label{fig:sim_ratio}
     \end{subfigure}
    \caption{\small \textbf{(a) Error diversity predicts RLVR effectiveness.} We find that Improvement Rate---the ratio of problems showing increased Average Pass Rate (APR)---positively correlates with initial error-set diversity. We collect the error-prone set (initial APR < 100\%) and the hardest subsets (initial APR = 0\%), and partition them equally into four quartiles according to error diversity. The results suggest that higher error diversity provides a richer learning signal for RLVR. See \autoref{ap:des4moti} for details. \textbf{(b) Shared wrong answers indicate shared reasoning roots.} We compare the probability of sharing a fundamental logical error between trajectories yielding the same wrong answer ($y$-axis) versus different wrong answers ($x$-axis). With 94\% of evaluated problems lying above the $y=x$ diagonal, the results demonstrate that identical incorrect answers predominantly stem from similar reasoning paths, whereas differing wrong answers reflect distinct logical failures. See \autoref{ap:des4moti2} for details. \textbf{(c) Shared wrong answers indicate shared language pattern.}  The ratio is calculated as within-group similarity divided by between-group similarity for each problem. Points above the dashed line indicate that models generate more consistent 3-gram language pattern when producing the same incorrect answer. This systematic trend persists across all tested AIME, AMC, and HMMT benchmarks.}\label{fig:motivation}
    \vspace{-0.5em}
\end{figure}

\section{Motivation}\label{sec:motivation}

In this section, we present empirical evidence that the diversity of incorrect answers within a rollout group is a strong predictor of RLVR training effectiveness, and validate that distinct wrong answers genuinely reflect distinct underlying reasoning paths.

\subsection{Error Diversity Predicts RLVR Effectiveness}

Our empirical investigation reveals a striking pattern (\autoref{fig:uprate}): the \textbf{diversity of wrong answers within a rollout group} is a strong predictor of whether RLVR training will improve the model’s performance on that problem---problems with diverse errors exhibit substantially higher improvement rates, whereas those with homogeneous failures see limited gains.

This observation rests on a key premise: that distinct wrong answers genuinely reflect distinct underlying reasoning paths, rather than mere surface variation. We provide converging evidence from two independent angles. Using an LLM judge to assess logical roots, \autoref{fig:logic_similarity} shows that trajectories yielding the same wrong answer predominantly share the same fundamental reasoning error, while those yielding different wrong answers reflect distinct logical failures. This is independently confirmed at the language pattern level: \autoref{fig:sim_ratio} shows trajectories with the same incorrect answer exhibit significantly higher 3-gram similarity in their reasoning steps than those with different answers. A detailed case in \autoref{fig:case_study} further establishes that wrong answer diversity is a reliable proxy for reasoning-path diversity, explaining why it predicts RLVR effectiveness and motivating the exploitation of intra-group error structure.

\subsection{From Observation to Method}

The above findings suggest that the intra-group error distribution carries rich, actionable information that current RLVR algorithms discard. A natural strategy is to reshape the advantage signal so that dominant, repeated errors receive amplified penalties---pushing the model away from entrenched failure modes---while rare, exploratory errors receive attenuated penalties, preserving the model’s diverse reasoning attempts. We formalize this intuition as \textbf{Error Diversity Advantage Shaping (EDAS)} in the next section.

\section{Method}

In this section, we describe Error Diversity Advantage Shaping (EDAS), a lightweight advantage adjustment module that can be applied on top of any group-relative RLVR algorithm. EDAS modifies only the advantage values assigned to incorrect trajectories; correct trajectories and the rest of the training pipeline remain untouched.

\begin{figure}[t]
\centering
\includegraphics[width=0.84\linewidth]{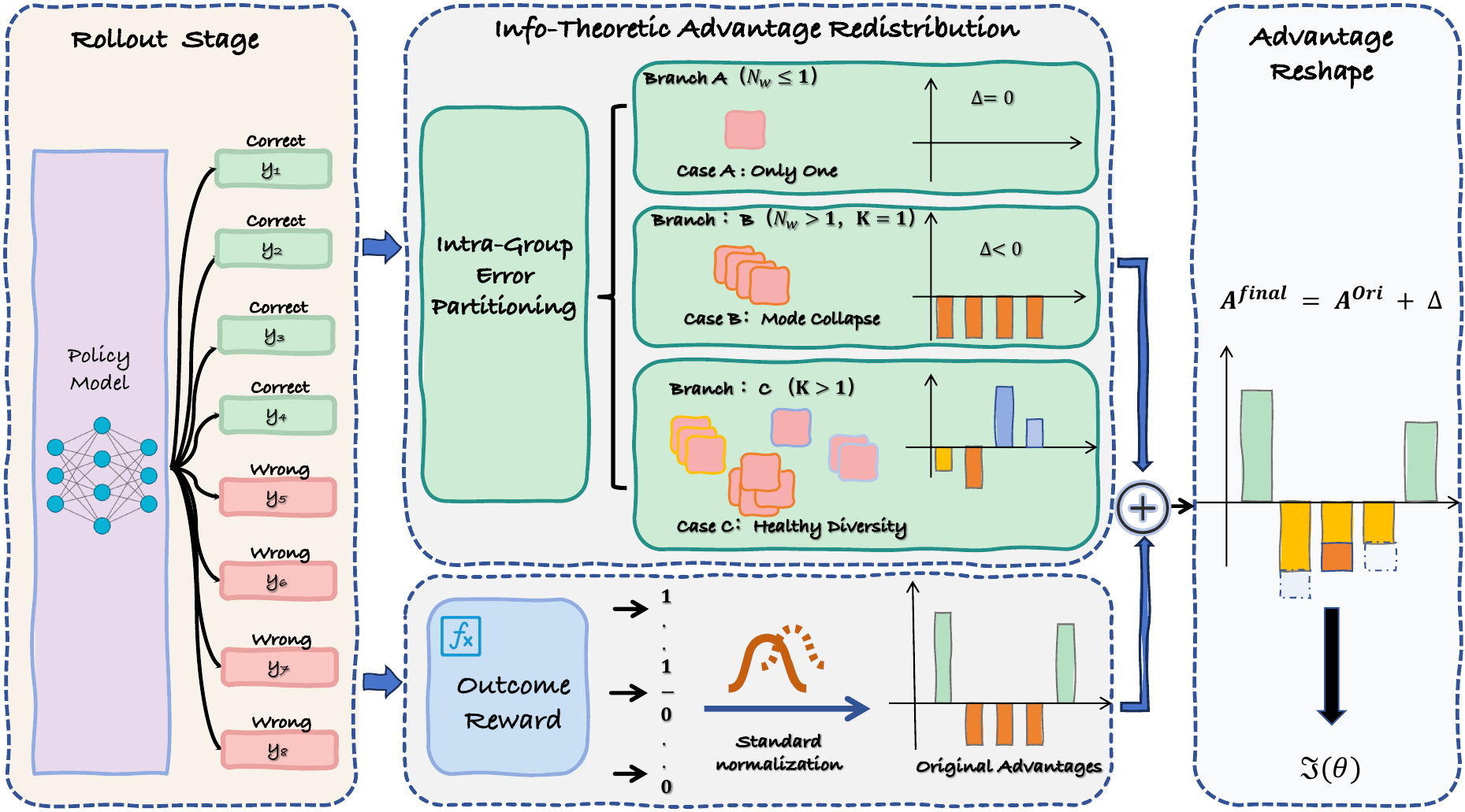}
\caption{Overview of EDAS. \textbf{(Left)} The policy model rolls out a group of trajectories for problem $x$, labeled correct/incorrect by the outcome reward. \textbf{(Middle)} Incorrect trajectories are partitioned into error classes by a domain-specific function $\mathcal{E}(\cdot)$, and an adjustment $\Delta$ is computed along three branches: (A) $N_w\!\le\!1$, $\Delta=0$; (B) mode collapse ($K\!=\!1$), uniform negative shift; (C) healthy diversity ($K\!>\!1$), redistribute penalties by normalized relative surprisal. \textbf{(Right)} The final advantage $A^{\mathrm{final}}=A^{\mathrm{ori}}+\Delta$ preserves the baseline's sign and intra-subset ordering, and feeds the policy gradient objective $\mathcal{J}(\theta)$.}
\label{fig:framework}
\end{figure}

\subsection{Intra-Group Error Partitioning}\label{sec:partition}

Let $\mathcal{W} \subseteq \{1, \dots, N\}$ denote the index set of trajectories that yield incorrect answers, with $N_w = |\mathcal{W}|$. Rather than committing to a fixed notion of error similarity, EDAS is built around a \textbf{domain-specific error equivalence function} $\mathcal{E}: \mathcal{Y} \to \mathcal{L}$, which maps each trajectory $y_i \in \mathcal{Y}$ to a discrete \emph{error label} $\ell_i = \mathcal{E}(y_i)$ drawn from a label set $\mathcal{L}$. Two incorrect trajectories are treated as belonging to the same error class if and only if they share the same label:
\[
    y_i \sim y_j \;\iff\; \mathcal{E}(y_i) = \mathcal{E}(y_j).
\]
The choice of $\mathcal{E}$ is the sole domain-specific design decision in EDAS; all downstream computation---entropy estimation, advantage redistribution, and clipping---is identical regardless of domain. Concretely, we describe the designs for the two most common reasoning scenarios:

\begin{itemize}
    \item \textbf{Mathematical reasoning.} $\mathcal{E}(y_i)$ extracts the canonicalized numerical answer from the \texttt{\textbackslash boxed\{\}} environment and applies symbolic normalization via the Math-Ruler library. Two trajectories are co-classified only if their extracted values are provably numerically equivalent. We show in \autoref{ap:embedding_failure} that embedding-based alternatives~\cite{chen2025dra,li2026setpo,wei2026mmr} are fundamentally unsuitable for this purpose.
    \item \textbf{Code generation.} $\mathcal{E}(y_i)$ maps each incorrect trajectory to its Python exception type (e.g., \texttt{SyntaxError}, \texttt{TypeError}, \texttt{WrongAnswer}), obtained by compiling and executing the generated code in a sandboxed environment against test cases (detailed in \autoref{sec:code_gen}). Exception types provide a natural semantic proxy for the underlying algorithmic failure mode.
\end{itemize}

This relation partitions the set of incorrect trajectories into $K$ disjoint equivalence classes (i.e., unique error types), denoted as $\mathcal{C} = \{C_1, C_2, \dots, C_K\}$, such that $\sum_{k=1}^K |C_k| = N_w$. For each class $C_k$, we define its empirical intra-group probability as $p_k = |C_k| / N_w$; for any trajectory $i \in C_k$, we use $p_k$ as its class probability.

To ensure that EDAS adjustments scale proportionally with the current learning dynamics---avoiding over-adjustment when advantages are small (e.g., early training) or under-adjustment when they are large---we define a dynamic scaling factor $S$ as the mean absolute advantage of the incorrect subset:

\begin{equation}
    S = \frac{1}{N_w} \sum_{i \in \mathcal{W}} \left| A^{\mathrm{orig}}_i \right|
\end{equation}

\subsection{Information-Theoretic Advantage Redistribution}

The core idea of EDAS is to penalize homogeneous failure modes while attenuating penalties for diverse, exploratory errors. We formalize this using the self-information of each trajectory $I_i = -\ln p_k$ where $C_k$ is the class containing $y_i$ and the empirical Shannon entropy of the group's error distribution $H = -\sum_{k=1}^K \frac{|C_k|}{N_w} \ln \left( \frac{|C_k|}{N_w} \right)$.

Depending on the cardinality of the error partition $K$, we formulate the advantage adjustment $\Delta_i$ for each $i \in \mathcal{W}$ as a piecewise function:

\begin{equation}
    \Delta_i =
\begin{cases}
0, & \text{if } N_w \le 1 \\
  -\beta \cdot S, & \text{if } N_w > 1 \text{ and } K = 1 \\
 \alpha \cdot S \cdot \left( \frac{I_i - H}{\ln N_w} \right), & \text{if } K > 1
\end{cases}
\end{equation}

Here, $\alpha, \beta > 0$ control the strength of diversity encouragement and collapse penalty.

The three branches capture distinct regimes:

\begin{itemize}
  \item \textbf{Insufficient Statistics ($N_w \le 1$)}: Too few errors to estimate diversity; the original advantage is preserved.
  \item \textbf{Error Perseveration ($K = 1$)}: All wrong answers are
  identical---the policy concentrates all errors on a single answer.
  EDAS applies a uniform negative shift $-\beta \cdot S$ to all
  incorrect trajectories, amplifying their negative advantage. This
  increases the magnitude of log-probability decrease for these tokens
  during the policy gradient update, accelerating the model's departure
  from the entrenched error and freeing probability mass for alternative
  reasoning paths.
  \item \textbf{Healthy Diversity ($K > 1$)}: Multiple error types coexist. EDAS redistributes penalties based on normalized relative surprisal $T_i = (I_i - H) / \ln N_w$, which is bounded in $[-1, 1]$ (see Appendix~\ref{ap:proof}). Frequent errors ($T_i < 0$) receive amplified penalties; rare errors ($T_i > 0$) receive attenuated penalties.
\end{itemize}

In the healthy diversity regime, this redistribution satisfies a strict zero-sum constraint: $\sum_{i \in \mathcal{W}} T_i = 0$ as proved in \autoref{ap:proof}. Consequently, EDAS reallocates the advantage signals without altering their arithmetic mean, preserving the unbiased nature of the base estimator. We summarize the advantage redistribution of EDAS in Algorithm \ref{alg:edpo_subroutine}.

\begin{algorithm}[t]
\caption{EDAS Advantage Redistribution Subroutine (ARS)}
\label{alg:edpo_subroutine}
\begin{algorithmic}[1]
\Require Incorrect trajectories subset $\mathcal{W}$, Baseline advantages $\{A^{\mathrm{orig}}_i\}_{i \in \mathcal{W}}$.
\Require Hyperparameters: $\alpha$ (diversity), $\beta$ (collapse penalty).
\Ensure Advantage adjustments $\{\Delta_i\}_{i \in \mathcal{W}}$.
\State $N_w \gets |\mathcal{W}|$
\If{$N_w \le 1$}
    \State \Return $\Delta_i = 0 \quad \forall i \in \mathcal{W}$ \Comment{Branch A: Insufficient samples}
\EndIf
\State Compute dynamic scale $S \gets \frac{1}{N_w} \sum_{i \in \mathcal{W}} |A^{\mathrm{orig}}_i|$
\State Partition $\mathcal{W}$ into $K$ equivalence classes $\mathcal{C} = \{C_1, \dots, C_K\}$
\If{$K = 1$}
    \State \Return $\Delta_i = - \beta \cdot S \quad \forall i \in \mathcal{W}$ \Comment{Branch B: Error perseveration penalty}
\Else
    \State Compute empirical probabilities $p_k \gets |C_k| / N_w$ for $k = 1, \dots, K$
    \State Compute subset entropy $H \gets - \sum_{k=1}^K p_k \ln p_k$
    \For{each error trajectory $i \in \mathcal{W}$} \Comment{Branch C: Healthy diversity}
        \State Find class $C_k$ such that $i \in C_k$
        \State Compute self-information $I_i \gets -\ln p_k$
        \State Calculate normalized relative surprisal $T_i \gets (I_i - H) / \ln N_w$
        \State $\Delta_i \gets \alpha \cdot S \cdot T_i$
    \EndFor
    \State \Return $\{\Delta_i\}_{i \in \mathcal{W}}$
\EndIf
\end{algorithmic}
\end{algorithm}

\subsection{Monotonicity-Preserving Clipping}\label{sec:clipping}

To prevent a positive adjustment from inverting the sign of $A^{\mathrm{orig}}_i$, which would falsely reward an incorrect trajectory, we clip the adjustment magnitude to $|A^{\mathrm{orig}}_i| / \kappa$ for a margin $\kappa > 1$:

\begin{equation}
    A^{\mathrm{final}}_i = A^{\mathrm{orig}}_i + \text{sgn}(\Delta_i) \cdot \min \left( |\Delta_i|, \frac{|A^{\mathrm{orig}}_i|}{\kappa} \right)
\end{equation}

This clipping mechanism guarantees no sign inversion of the advantage for incorrect trajectories, ensuring they always receive non-positive gradient updates, with the full formal proof provided in \autoref{ap:Clipping}.



\FloatBarrier
\section{Experiment}




In this section, we evaluate the effectiveness of EDAS across two challenging domains: mathematical reasoning and code generation. We aim to examine whether EDAS consistently improves RLVR performance across different model scales and task types.

\subsection{Experimental Setup}

\paragraph{Model.}

To evaluate the scalability and robustness of EDAS, we employ the Qwen3 series \cite{yang2025qwen3} as our base models, including Qwen3-8B, Qwen3-4B, and Qwen3-4B-Base. These models represent a diverse range of pre-existing reasoning proficiencies and parameter scales.

\paragraph{Implementation Details.}

We evaluate our method on mathematical reasoning and code generation tasks. Experiments were executed on a system utilizing 8× NVIDIA H200 GPUs. All models are trained under a standard RLVR framework where the reward is derived solely from the correctness of the final output. We maintain a consistent rollout configuration across both domains, utilizing a rollout batch size of 256 with 10 samples per prompt. The EDAS hyperparameters are fixed at $\alpha=0.4, \beta=0.2, \kappa=2.0$ for all experiments. For the mathematical reasoning tasks, we evaluate performance across seven mainstream benchmarks, including AIME (2024, 2025, 2026)~\cite{aime2024,aime2025,aime2026}, AMC (2023, 2024)~\cite{AMC}, HMMT 2025~\cite{balunovic2025matharena}, and OlympiadBench~\cite{he2024olympiadbench}. To assess code generation capabilities, we utilize four representative benchmarks: LiveCodeBench~\cite{jain2024livecodebench}, Codeforces~\cite{quan2025codeelo}, HumanEval+~\cite{chen2021evaluating,liu2023your}, and MBPP+~\cite{austin2021program,liu2023your}. Detailed task-specific configurations, including dataset statistics, training epochs, and evaluation settings, are provided in \autoref{ap:exp_det}.

\paragraph{Baselines.}

We compare EDAS against the standard RLVR algorithm GRPO~\cite{shao2024deepseekmath} and its comparable variant DAPO~\cite{yu2025dapo}. We also include PKPO~\cite{pkpo} as a baseline, which attempts to directly optimize the Pass@$k$ metric by theoretically leveraging intra-group information inherent in RLVR. Furthermore, we compare our method with entropy-based~\cite{cheng2026reasoning} regularization techniques (En Adv) that promote model exploration by augmenting the reward signal. Notably, as entropy-based exploration encountered training instability on larger models, we report its peak performance sampled at 10-step intervals to ensure a fair comparison. Detailed hyperparameter configurations for these additional baselines are provided in \autoref{ap:base_setup}.



\begin{figure}[t]
    \centering
    \includegraphics[width=\linewidth]{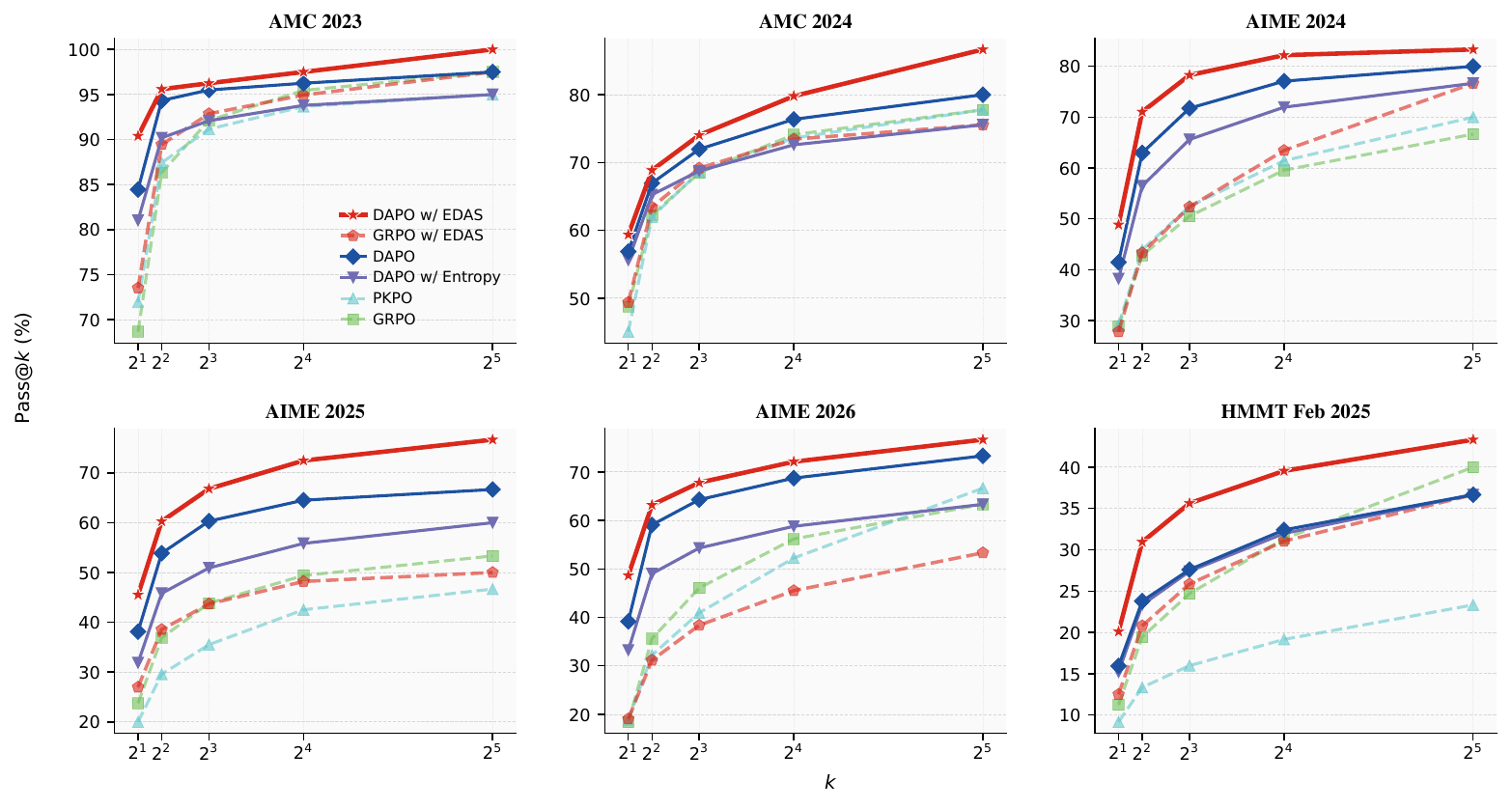}
    \caption{\small Pass@$k$ curves on six benchmarks for Qwen3-8B. EDAS w/ DAPO consistently achieves the highest Pass@$k$ across all values of $k$. The advantage is especially pronounced at larger $k$ on harder benchmarks, where diversity in reasoning paths directly translates to higher coverage of correct solutions.}
    \label{fig:passk}
\end{figure}

\subsection{Math Reasoning}

\begin{table}[t]
\centering
\caption{\small \textbf{Main Results on Math Benchmarks.} Comparison of our \textbf{EDAS} enhancement across three Qwen3 models. $\Delta$ denotes the improvement of our best variant (\textbf{EDAS w/ DAPO}) over the strongest existing baseline. The best results are \textbf{bolded}, and our methods are highlighted in \colorbox{lightred}{light red}. Due to the significant increase in inference demands resulting from the use of Dynamic Sampling, we evaluate DAPO as a separate comparison.}
\label{tab:edpo_full_results}
\resizebox{\textwidth}{!}{%
\begin{tabular}{l@{\hspace{12pt}}cccccccc}
\toprule
\textbf{Method} & \textbf{AMC23} & \textbf{AMC24} & \textbf{AIME24} & \textbf{AIME25} & \textbf{AIME26} & \textbf{HMMT} & \textbf{Olympiad} & \textbf{Avg.} \\
\midrule
\rowcolor[gray]{.95} \multicolumn{9}{l}{\textit{Qwen3-8B}} \\
Baseline & 67.03 & 44.24 & 23.85 & 20.00 & 15.10 & 10.42 & 54.75 & 33.63 \\
+ GRPO   & 68.67 & 48.75 & 28.85 & 23.75 & 18.43 & 11.25 & 57.57 & 36.75 \\
+ PKPO   & 71.95 & 45.07 & \textbf{29.16} & 20.00 & 18.85 & 12.08   & 58.90 & 36.57 \\
\rowcolor{lightred} + \textbf{GRPO w/ EDAS} & \textbf{73.52} & \textbf{49.38} & 27.81 & \textbf{26.98} & \textbf{19.06} & \textbf{12.50} & \textbf{59.20} & \textbf{38.35} \\
{$\Delta$} & \textcolor[HTML]{006400}{+1.57} & \textcolor[HTML]{006400}{+0.63} & \textcolor[HTML]{B22222}{-1.35} & \textcolor[HTML]{006400}{+3.23} & \textcolor[HTML]{006400}{+0.21} & \textcolor[HTML]{006400}{+0.42} & \textcolor[HTML]{006400}{+0.30} & \textcolor[HTML]{006400}{+1.60} \\

\lightrule
+ DAPO   & 84.45 & 56.88 & 41.46 & 38.12 & 39.17 & 15.94 & 58.75 & 47.82 \\
+ DAPO w/ En Adv & 81.02 & 55.62 & 31.81 & 38.23 & 33.23 & 15.21 & 58.31 & 44.78 \\
\rowcolor{lightred} + \textbf{DAPO w/ EDAS} & \textbf{90.39} & \textbf{59.38} & \textbf{48.85} & \textbf{45.52} & \textbf{48.65} & \textbf{20.10} & \textbf{65.88} & \textbf{54.11} \\
{$\Delta$} & \textcolor[HTML]{006400}{+5.94} & \textcolor[HTML]{006400}{+2.50} & \textcolor[HTML]{006400}{+7.39} & \textcolor[HTML]{006400}{+7.29} & \textcolor[HTML]{006400}{+9.48} & \textcolor[HTML]{006400}{+4.16} & \textcolor[HTML]{006400}{+7.13} & \textcolor[HTML]{006400}{+6.29} \\

\midrule
\rowcolor[gray]{.95} \multicolumn{9}{l}{\textit{Qwen3-4B}} \\
Baseline & 67.03 & 44.31 & 22.08 & 20.01 & 18.02 & 10.83 & 53.41 & 33.68 \\
+ DAPO   & 85.70 & 57.15 & \textbf{47.50} & 43.54 & 40.73 & 22.40 & 65.88 & 51.84 \\
+ DAPO w/ En Adv & 77.27 & 49.58 & 34.58 & 31.87 & 26.77 & 13.65 & 59.35 & 41.87 \\
\rowcolor{lightred} + \textbf{DAPO w/ EDAS} & \textbf{88.59} & \textbf{58.19} & 46.46 & \textbf{48.12} & \textbf{44.27} & \textbf{23.12} & \textbf{67.36} & \textbf{53.73} \\
{$\Delta$} & \textcolor[HTML]{006400}{+2.89} & \textcolor[HTML]{006400}{+1.04} & \textcolor[HTML]{B22222}{-1.04} & \textcolor[HTML]{006400}{+4.58} & \textcolor[HTML]{006400}{+3.54} & \textcolor[HTML]{006400}{+0.72} & \textcolor[HTML]{006400}{+1.48} & \textcolor[HTML]{006400}{+1.89} \\

\midrule
\rowcolor[gray]{.95} \multicolumn{9}{l}{\textit{Qwen3-4B-Base}} \\
Baseline & 31.56 & 12.08 & 10.31 & 5.52 & 2.60 & 1.56 & 26.85 & 12.93 \\
+ DAPO   & 60.00 & 33.68 & 12.81 & 18.02 & 10.73 & 8.33 & 46.44 & 27.14 \\
+ DAPO w/ En Adv & 55.08 & 34.24 & 16.98 & 17.40 & 12.71 & 6.67 & 45.25 & 26.90 \\
\rowcolor{lightred} + \textbf{DAPO w/ EDAS} & \textbf{72.73} & \textbf{48.68} & \textbf{24.17} & \textbf{25.21} & \textbf{20.97} & \textbf{10.94} & \textbf{56.50} & \textbf{37.02} \\
{$\Delta$} & \textcolor[HTML]{006400}{+12.73} & \textcolor[HTML]{006400}{+14.44} & \textcolor[HTML]{006400}{+7.19} & \textcolor[HTML]{006400}{+7.19} & \textcolor[HTML]{006400}{+8.26} & \textcolor[HTML]{006400}{+2.61} & \textcolor[HTML]{006400}{+10.06} & \textcolor[HTML]{006400}{+9.88} \\
\bottomrule
\end{tabular}%
}
\end{table}

\paragraph{Results.}

As illustrated in \autoref{tab:edpo_full_results}, applying EDAS consistently improves reasoning capabilities across all three models when combined with mainstream reinforcement learning algorithms. Since DAPO employs Dynamic Sampling—which filters out samples where the model is entirely correct or entirely incorrect and triggers re-sampling—it consumes significantly more computational resources for the same number of steps. Dynamic Sampling also plays a critical role in EDAS, as discussed in \autoref{ap:Clipping}. Consequently, we conduct separate comparisons based on the use of Dynamic Sampling. Over the strongest baseline (DAPO), EDAS achieves average improvements of \textbf{6.29} points on Qwen3-8B, \textbf{1.89} points on Qwen3-4B, and \textbf{9.88} points on Qwen3-4B-Base across seven benchmarks, demonstrating the substantial effectiveness of leveraging error diversity information in group rollouts.

\paragraph{Pass@$k$ Results.}
Beyond average accuracy, we evaluate Pass@$k$ to measure the model's coverage of correct solutions across multiple attempts. \autoref{fig:passk} reports Pass@$k$ for $k \in \{2, 4, 8, 16, 32\}$ on six benchmarks using Qwen3-8B. EDAS w/ DAPO consistently dominates all baselines across all values of $k$, confirming that the improvement is not limited to greedy decoding but extends to the model's ability to produce at least one correct solution among $k$ attempts. Notably, the gap between EDAS and DAPO widens as $k$ increases on harder benchmarks (AIME 2025, AIME 2026, HMMT), suggesting that EDAS's diversity-promoting mechanism enriches the pool of distinct reasoning paths, making it increasingly likely that at least one succeeds. EDAS w/ GRPO similarly outperforms both GRPO and PKPO, further confirming the generality of the approach.

\subsection{Code Generation}\label{sec:code_gen}

To demonstrate that EDAS generalizes beyond mathematical reasoning, we evaluate it on code generation tasks. The key adaptation lies in the definition of the equivalence relation for error partitioning: instead of grouping incorrect responses by their final numerical answer, we group them by Python exception type, the implementation details of which are provided in \autoref{ap:code_gene}.

\paragraph{Results.}

As shown in \autoref{tab:code_results}, EDAS consistently improves over DAPO across all four benchmarks, with an average gain of \textbf{1.45} points. The improvement is most pronounced on Codeforces, which contains algorithmically challenging problems where diverse error exploration is particularly valuable---a model stuck on a single algorithmic approach benefits from EDAS pushing it toward alternative strategies that may yield different failure modes but also open paths to correct solutions.

These results confirm two important properties of EDAS: (1) the core principle---that error diversity is an informative signal for RLVR---extends beyond mathematical reasoning to code generation, and (2) the framework is flexible enough to accommodate domain-specific definitions of error equivalence while preserving the same advantage redistribution mechanism.

\begin{table}[htbp]
\centering
\caption{\small Code generation results on Qwen3-4B. $\Delta$ denotes the improvement of EDAS over DAPO.}
\label{tab:code_results}
\begin{tabular}{lccccc}
\toprule
\textbf{Method} & \textbf{LiveCodeBench} & \textbf{Codeforces} & \textbf{HumanEval+} & \textbf{MBPP+} & \textbf{Avg.} \\
\midrule
+ DAPO & 31.21 & 42.42 & 79.88 & 63.23 & 54.19 \\
\rowcolor{lightred} + \textbf{DAPO w/ EDAS} & \textbf{32.07} & \textbf{45.97} & \textbf{80.49} & \textbf{64.02} & \textbf{55.64} \\
{$\Delta$} & \textcolor[HTML]{006400}{+0.86} & \textcolor[HTML]{006400}{+3.55} & \textcolor[HTML]{006400}{+0.61} & \textcolor[HTML]{006400}{+0.79} & \textcolor[HTML]{006400}{+1.45} \\
\bottomrule
\end{tabular}
\end{table}

\FloatBarrier
\section{Analysis}


\paragraph{EDAS Promotes Error Diversity and Enhances RLVR Performance.}

\autoref{fig:meanerr} shows that integrating EDAS into DAPO sustains a higher variety of unique wrong answers during training. Furthermore, EDAS consistently achieves superior reward levels as shown in \autoref{fig:reward}. These dynamics indicate that EDAS prevents premature convergence to narrow failure modes, encouraging the exploration of diverse reasoning paths and ultimately boosting overall RLVR efficacy.

\begin{figure}[htbp]
    \centering
    \hfill
    \begin{subfigure}{0.49\textwidth}
        \centering
        \includegraphics[width=\linewidth]{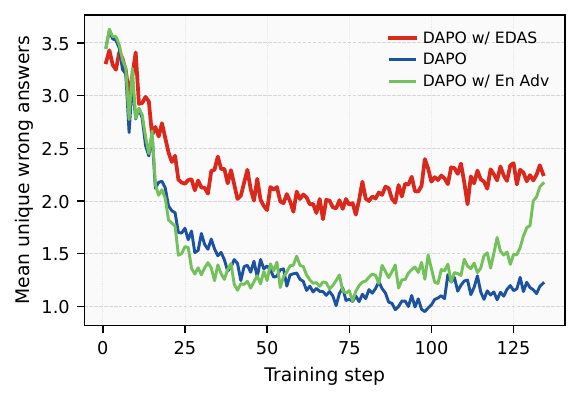}
        \caption{}
        \label{fig:meanerr}
    \end{subfigure}
    \hfill
    \begin{subfigure}{0.49\textwidth}
        \centering
        \includegraphics[width=\linewidth]{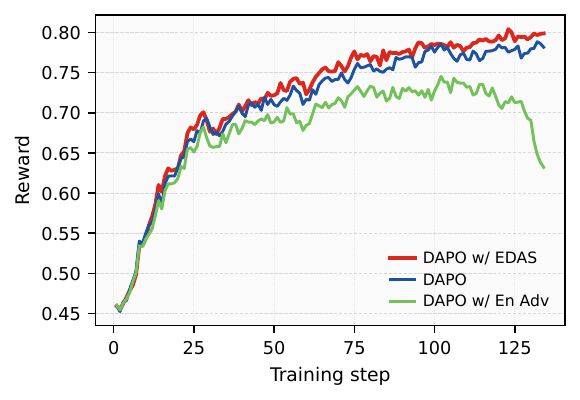}
        \caption{}
        \label{fig:reward}
    \end{subfigure}
    \caption{\small (a) EDAS preserves error diversity throughout training. The trajectories of mean unique wrong answers show that while DAPO's error diversity drops rapidly, applying EDAS maintains a higher variety of failure modes, underpinning enhanced capability to discover successful reasoning paths. (b) EDAS improves optimization dynamics. Reward curves (0.5 exponential smoothing) show that EDAS achieves a higher asymptotic reward and faster growth than DAPO and DAPO w/ Entropy Advantage.}
    \label{fig:training_dynamics}
\end{figure}



\paragraph{EDAS Facilitates Breakthroughs on Capability-Frontier Problems.}

EDAS further demonstrates its utility on capability-frontier problems. Solving these instances requires the discovery of genuinely novel reasoning paths through RL. \autoref{tab:bottleneck_8b} summarizes the "breakthrough" rates for Qwen3-8B; across four benchmarks, 64 problems are identified as hard. EDAS successfully resolves 30 of these, compared to 22 for DAPO, representing a 36\% relative improvement. Notably, EDAS uniquely solves 10 problems that DAPO fails to crack entirely (\text{APR} = 0\%). Detailed analysis is provided in \autoref{ap:bottleneck}.

\begin{table}[htbp]
\centering
\caption{\small \textbf{Bottleneck-breaking analysis on Qwen3-8B.} \emph{Hard} denotes problems where the base model achieved zero accuracy over 32 rollouts. \emph{Breakthrough} indicates the number of such problems solved ($\text{APR} > 0$) after RL training. \emph{EDAS-only} highlights the unique capacity of our method to solve instances that remains unsolved by DAPO.}
\label{tab:bottleneck_8b}
\small 
\begin{tabular}{l@{\hspace{20pt}}cccc}
\toprule
\textbf{Benchmark} & \textbf{Hard (Total)} & \textbf{DAPO Break} &  \textbf{EDAS Break} & \textbf{EDAS-only} \\
\midrule
\rowcolor[gray]{.95} \multicolumn{5}{l}{\textit{Qwen3-8B (Zero-Shot Hard Subset)}} \\
AIME 2025    & 15 &  6 & \textbf{9} & 3 \\
AIME 2026    & 15 &  7 & \textbf{8} & 1 \\
AMC 2024      & 13 &  5 & \textbf{7} & 3 \\
HMMT Feb 25  & 21 &  4 & \textbf{6} & 3 \\
\midrule
\textbf{Total Sum} & \textbf{64} & \textbf{22} & \textbf{30} & \textbf{10} \\
\textbf{Success Rate} & 0.0\% & 34.4\% & \textbf{46.9\%} & \textcolor[HTML]{006400}{+12.5\%} \\
\bottomrule
\end{tabular}
\end{table}


\vspace{-2mm}

\FloatBarrier
\section{Related Work}
\paragraph{Group-Level Information in Rollouts.}
Recent RLVR research~\cite{xie2026step} increasingly exploits the structured information within rollout groups. PKPO~\cite{pkpo} optimizes Pass@$k$ coverage via low-variance gradient estimators, while RSPO~\cite{zhang2025rspo} and Harder-is-Better~\cite{dai2026harder} reweight gradients based on trajectory rarity or problem difficulty. To enhance set-level dynamics, SetPO~\cite{li2026setpo} introduces marginal contribution bonuses, and NGRPO~\cite{nan2025ngrpo} addresses gradient vanishing in homogeneous failed groups by injecting negative signals. Distinct from these, EDAS is motivated by a novel empirical observation: the diversity of incorrect answers within a group is a strong predictor of training success, as distinct errors reflect broader exploration across reasoning paths.

\paragraph{Diversity of Reasoning Paths in RLVR.}
To combat policy entropy collapse, advantage-shaping methods such as UCAS~\cite{xie2025unlocking}, EDGE-GRPO~\cite{zhang2025edge}, and ProGRPO~\cite{bu2025post} modulate rewards to penalize overconfidence or utilize internal probability structures. Alternatively, structural approaches intervene via tree-structured branching~\cite{zhao2026reinforced}, noise embeddings~\cite{lyudiversity}, or direct CoT entropy integration~\cite{hao2025reasoning}. Unlike these methods that modify sampling, architectures, or the full distribution, EDAS provides a post-hoc, complementary adjustment. It targets the error structure specifically, using the empirical surprisal of incorrect trajectories to differentially reweight advantages—a lightweight mechanism that leaves all other training components unchanged.

\vspace{-3mm}
\section{Conclusion}

Our empirical analysis shows that intra-group error diversity is strongly correlated with RLVR effectiveness, suggesting that diverse failure modes can provide useful learning signals beyond binary correctness. Building on this observation, we introduce Error Diversity Advantage Shaping (EDAS), a lightweight advantage-shaping method that redistributes penalties among incorrect rollouts according to their error frequency. Across multiple Qwen3 variants and two domains (math and code), EDAS improves performance over the corresponding baselines and helps maintain error diversity during training. Promising future directions include incorporating diversity signals from intermediate reasoning steps and extending EDAS to other verifiable domains. We discuss limitations in \autoref{ap:limitations}.




\clearpage
\bibliographystyle{plainnat}
\bibliography{cite}

\newpage
\beginappendix

\section{Proof of Zero-Sum Property in the Diverse Regime}\label{ap:proof}

\begin{theorem}[Zero-Sum Advantage Redistribution]
In the healthy diversity regime ($K > 1$), the proposed Error Diversity Advantage Shaping (EDAS) guarantees that the arithmetic mean of the normalized relative surprisal $T_i$ over all incorrect trajectories is exactly zero, i.e., $\frac{1}{N_w} \sum_{i \in \mathcal{W}} T_i = 0$.
\end{theorem}

 \begin{proof}
     Recall the definition of the normalized relative surprisal for the $i$-th incorrect trajectory:
     \begin{equation}
         T_i = \frac{I_i - H}{\ln N_w}
     \end{equation}

     where $I_i = -\ln p_k$ is the self-information of trajectory $i$ with $C_k$ denoting the class containing $y_i$, and $H$ is the empirical Shannon entropy of the group's error distribution. We are interested in the expected adjustment across the subset of incorrect trajectories $\mathcal{W}$, with $|\mathcal{W}| = N_w$.

     By substituting $T_i$ into the summation, we have:
     \begin{equation}
         \frac{1}{N_w} \sum_{i \in \mathcal{W}} T_i = \frac{1}{N_w \ln N_w} \sum_{i \in \mathcal{W}} (I_i - H)
     \end{equation}

     To evaluate $\sum_{i \in \mathcal{W}} I_i$, we can group the individual trajectories by their canonical equivalence classes $\mathcal{C} = \{C_1, C_2, \dots, C_K\}$. For any trajectory $i \in C_k$, its empirical probability is defined as $p_k = \frac{|C_k|}{N_w}$. Summing the self-information over all incorrect trajectories is equivalent to summing over the equivalence classes weighted by their frequencies $|C_k|$:

     \begin{equation}
         \sum_{i \in \mathcal{W}} I_i = \sum_{k=1}^K |C_k| \cdot \left( -\ln \frac{|C_k|}{N_w} \right)
     \end{equation}

     Substituting $|C_k| = N_w \cdot p_k$ into the equation yields:

     \begin{equation}
         \sum_{i \in \mathcal{W}} I_i = \sum_{k=1}^K N_w \cdot p_k (-\ln p_k) = N_w \sum_{k=1}^K -p_k \ln p_k
     \end{equation}

     By the definition of the empirical Shannon entropy $H$, the summation term $\sum_{k=1}^K -p_k \ln p_k$ is exactly $H$. Thus, we obtain:

     \begin{equation}
         \sum_{i \in \mathcal{W}} I_i = N_w H
     \end{equation}

     Substituting this back into our original equation:

     \begin{equation}
         \frac{1}{N_w} \sum_{i \in \mathcal{W}} (I_i - H) = \frac{N_w H - N_w H}{N_w} = 0
     \end{equation}

     Hence, $\frac{1}{N_w} \sum_{i \in \mathcal{W}} T_i = 0$. 
 \end{proof}

This property guarantees that the EDAS adjustment $\Delta_i \propto T_i$
operates strictly as an internal redistribution of advantage within the
incorrect subset. It reshapes the gradient landscape to favor minoritarian
hypotheses (exploration) and suppress majoritarian traps (error perseveration), without shifting the global baseline of the base RL algorithm.

\paragraph{Remark: Effect of $\kappa$-Clipping on the Zero-Sum Property.}
It is important to note that the zero-sum constraint $\sum_{i \in \mathcal{W}} T_i = 0$ holds exactly \emph{before} the monotonicity-preserving $\kappa$-clipping step described in \autoref{sec:clipping}. In practice, the clipping operation
\[
    \Delta_{\mathrm{final},i} = \operatorname{sgn}(\Delta_i)\cdot\min\!\left(|\Delta_i|,\,\tfrac{|A^{\mathrm{orig}}_i|}{\kappa}\right)
\]
may truncate some adjustments asymmetrically---for instance, when a frequent-error trajectory has a very small original $|A^{\mathrm{orig}}_i|$, its penalty term may be truncated. This introduces a \emph{slight imbalance} in the redistributed advantages: the effective sum $\sum_{i \in \mathcal{W}} \Delta_{\mathrm{final},i}$ may deviate marginally from zero.

We emphasize that this deviation is an intentional \textbf{engineering trade-off} rather than a theoretical deficiency. The $\kappa$-clipping serves a critical purpose: it prevents any single adjustment from inverting the sign of an incorrect trajectory's advantage, which would spuriously reward a wrong answer and destabilize policy updates. In practice, the imbalance induced by clipping is negligibly small---the clipping constraint is rarely active for the majority of trajectories---and does not meaningfully compromise the unbiasedness of the base estimator. The zero-sum property thus characterizes the \emph{ideal} behavior of the redistribution mechanism; clipping constrains that redistribution only at extreme values where unconstrained adjustments would be harmful to training stability.

\section{Proof of the Sign Non-Inversion Property of Monotonicity-Preserving Clipping}\label{ap:Clipping}

For any incorrect trajectory, the base algorithm guarantees $A_i^{\mathrm{orig}} < 0$. Given clipping margin $\kappa > 1$, the final advantage is
\[
    A_i^{\mathrm{final}} = A_i^{\mathrm{orig}} + \operatorname{sgn}(\Delta_i)\cdot\min\!\left(|\Delta_i|,\,\frac{|A_i^{\mathrm{orig}}|}{\kappa}\right).
\]

\begin{theorem}[Sign Non-Inversion]
For any incorrect trajectory with $A_i^{\mathrm{orig}} < 0$ and $\kappa > 1$,
\begin{enumerate}
    \item $A_i^{\mathrm{final}} < 0$ \quad (no sign inversion),
    \item $|A_i^{\mathrm{final}}| \geq |A_i^{\mathrm{orig}}|\!\left(1 - \tfrac{1}{\kappa}\right) > 0$ \quad (strict lower bound on magnitude).
\end{enumerate}
\end{theorem}

\begin{proof}
Let $M = \min\!\left(|\Delta_i|,\,\tfrac{|A_i^{\mathrm{orig}}|}{\kappa}\right) \geq 0$.

\paragraph{Case 1: $\Delta_i \leq 0$ (penalty amplification).}
$A_i^{\mathrm{final}} = A_i^{\mathrm{orig}} - M \leq A_i^{\mathrm{orig}} < 0$, so (1) holds trivially. Moreover, $|A_i^{\mathrm{final}}| = |A_i^{\mathrm{orig}}| + M \geq |A_i^{\mathrm{orig}}| > |A_i^{\mathrm{orig}}|(1-\tfrac{1}{\kappa})$, so (2) holds.

\paragraph{Case 2: $\Delta_i > 0$ (penalty attenuation).}
$A_i^{\mathrm{final}} = A_i^{\mathrm{orig}} + M$. Since $M \leq \tfrac{|A_i^{\mathrm{orig}}|}{\kappa} < |A_i^{\mathrm{orig}}|$ (as $\kappa > 1$), the adjustment cannot overcome the negative base, so
\[
    A_i^{\mathrm{final}} = A_i^{\mathrm{orig}} + M < A_i^{\mathrm{orig}} + |A_i^{\mathrm{orig}}| = 0,
\]
establishing (1). For the magnitude bound, $|A_i^{\mathrm{final}}| = |A_i^{\mathrm{orig}}| - M \geq |A_i^{\mathrm{orig}}| - \tfrac{|A_i^{\mathrm{orig}}|}{\kappa} = |A_i^{\mathrm{orig}}|(1 - \tfrac{1}{\kappa}) > 0$, establishing (2).
\end{proof}

Hence, the clipping mechanism ensures that incorrect trajectories always carry a strictly negative advantage, preventing any spurious reward while bounding the residual gradient magnitude away from zero. In practice, while the clipping margin $\kappa$ safeguards the advantage sign, we recommend enabling Dynamic Sampling~\cite{yu2025dapo} when applying EDAS. This is particularly crucial in edge cases where all sampled sequences are incorrect, as it avoids potential scenarios where incorrect sequences might otherwise receive positive reward signals, ensuring stable policy updates across all group configurations.

\section{Correlation between Error Diversity and RLVR Gains}\label{ap:des4moti}

In this section, we give a detailed description of \autoref{fig:uprate}. To investigate the relationship between the baseline model's failure modes and the effectiveness of RLVR, we conduct a fine-grained analysis on Qwen3-8B-Instruct across six mainstream mathematical benchmarks: AIME (2024–2026), AMC (2023–2024), and HMMT Feb 2025.

\textbf{Experimental Setup:} We apply the state-of-the-art RLVR algorithm, DAPO, training the model on the DAPO-MATH-17K dataset for two epochs. We then analyze the performance shifts by comparing the checkpoints before and after the RLVR process.

\textbf{Metrics:} We define error diversity as the ratio of unique incorrect responses to the total number of incorrect trials during 32-sample repeated testing. The primary evaluation metric is the Improvement Rate, which tracks the portion of problems where the RLVR-tuned model achieves a higher Average Pass Rate (APR) compared to its predecessor.

\textbf{Grouping:} Problems are partitioned into four equal quartiles ($Q_1$ to $Q_4$) based on their diversity index. We specifically analyze two cohorts: (1) Error-prone problems, where the base model achieved an $APR < 100\%$ ; and (2) Hardest Subsets, representing the most challenging cases where the base model’s initial $APR$ was exactly $0\%$.

\textbf{Observations:} As shown in \autoref{fig:uprate}, there is a consistent upward trend in improvement across both cohorts as diversity increases. Notably, for problems in $Q_4$ (highest diversity), the improvement rate reaches approximately 80\%. The error bars in the plot further illustrate the specific range of diversity values within each quartile for the Error-prone set. This trend indicates that problems with multiple distinct failure modes offer more informative feedback during the RLVR process, leading to more robust policy optimization even on tasks where the model initially showed no signs of success.

\section{Validation of Final Answer as a Proxy for Reasoning Path Consistency}\label{ap:des4moti2}
To investigate the logical origins of model failures, we designed an evaluation pipeline to determine if different reasoning trajectories leading to the same incorrect answer share a fundamental logical root. We collected reasoning samples from six math competition benchmarks: AIME (2024–2026), AMC (2023–2024), and HMMT (Feb 2025), sampling each problem 32 times to generate a diverse set of reasoning paths and final answers.

The experimental pipeline first groups all incorrect trajectories for a given problem by their final answer strings. To analyze the internal consistency of these errors, we enumerate all possible pairs within each wrong-answer group (within-group pairs) and employ a judge model (Gemini 3 Flash Preview) to determine if both trajectories fail due to the same fundamental logical error, regardless of phrasing differences. For comparison, we establish a baseline by selecting a representative trajectory from each distinct wrong-answer group to evaluate pairs that produced different incorrect results (cross-group pairs).
We are using the prompt shown below:

\begingroup
\centering
\begin{tcolorbox}[
  colback=casered!6,
  colframe=casered!100,
  coltitle=white,
  fonttitle=\bfseries\normalsize,
  title={JUDGE\_PROMPT\_TEMPLATE},
  boxrule=0.8pt,
  arc=2pt,
  left=8pt, right=8pt, top=6pt, bottom=6pt,
  fontupper=\normalfont\small,
]
I will provide a math problem, its ground truth answer, and two incorrect reasoning trajectories (Trajectory A and Trajectory B).

\smallskip
Your task is to determine if both trajectories fail due to the \textbf{exact same fundamental logical error or misunderstanding}, or if their root causes of failure are \textbf{substantively different}.
Ignore differences in phrasing, variable naming, or formatting. Focus solely on the mathematical logic.

\smallskip
Problem: \texttt{\{problem\}} \\
Ground Truth: \texttt{\{ground\_truth\}} \\
Trajectory A: \texttt{\{traj\_a\}} \\
Trajectory B: \texttt{\{traj\_b\}} \\
\smallskip
Output JSON format:
\begin{verbatim}
{
  "is_same_error_root": true/false
}
\end{verbatim}
\end{tcolorbox}
\endgroup
\vspace{0.5em}

The resulting metrics—the within-group same-root rate ($y$-axis) and the cross-group same-root rate ($x$-axis)—are calculated as the fraction of same-root judgments within their respective categories. To ensure a robust comparison, we filter for questions that contain at least two distinct wrong-answer groups, with at least one group containing multiple trajectories, resulting in $N = 127$ valid problems. The strong tendency for points to cluster above the diagonal validates our hypothesis: shared wrong answers serve as a reliable proxy for shared reasoning error roots, indicating that the diversity of final answers effectively reflects the diversity of underlying logical failures.

\section{EDAS Implementation of Code Generation}\label{ap:code_gene}

For each incorrect rollout, the error classification follows a hierarchical process. We first identify compilation failures (e.g., \texttt{SyntaxError}), assigning the exception name as the label. Remaining rollouts are executed in a sandbox; those failing test cases are labeled \texttt{WrongAnswer}, while those triggering runtime exceptions are categorized by their specific exception types (e.g., \texttt{TypeError}, \texttt{IndexError}). This taxonomy naturally captures the diversity of model failures: a distribution spanning multiple runtime exceptions indicates a more diverse exploration of the task compared to a model limited to repetitive structural or syntactic errors.

\section{Experimental Details}\label{ap:exp_det}

This appendix provides a comprehensive summary of the hyperparameters and experimental configurations used for the mathematical reasoning and code generation tasks.

\subsection{Evaluation}

\paragraph{Math Reasoning.}

The maximum reasoning length for evaluation is set to 32,768 tokens. For datasets with limited samples (AIME 24-26, AMC 23-24, and HMMT 25), we employ a repeat-32 strategy with a sampling temperature of $T=0.7$, top-p of 0.95, and top-k of 40 to report the average accuracy. For OlympiadBench, we utilize greedy decoding to ensure deterministic results. To ensure the robustness of answer extraction and matching, we leverage the Math-Ruler library\footnote{https://github.com/hiyouga/MathRuler}. To assess how the model's capabilities evolve, we employ the unbiased estimator proposed by \cite{chen2021evaluating} to calculate Pass@$k$:

\begin{equation}
\text{Pass}@k = \mathbb{E}_{x \sim \mathcal{D}} \left[ 1 - \frac{\binom{n-c}{k}}{\binom{n}{k}} \right]
\end{equation}

where $n$ ($n \ge k$) denotes the total number of generated sequences per problem, and $c$ represents the number of correct sequences among them.

\paragraph{Code Generation.}
For code-generation evaluation, we set the maximum generation length to 16,384 tokens. We adopt benchmark-specific repeated sampling and report Pass@8 for LiveCodeBench v6 and Codeforces, and Pass@1 for HumanEval+, MBPP+. For repeated sampling, we use temperature $T=0.7$, top-$p=0.8$. A submission is considered correct only when it passes all test cases for a given problem. For robustness, code outputs are post-processed with benchmark-specific extraction and then validated through sandboxed execution or official evaluation protocols (e.g., EvalPlus for HumanEval+/MBPP+).

\subsection{Training Setup}

\paragraph{Math Reasoning.}

We utilized DAPO-Math-17K~\cite{yu2025dapo} as our training dataset, which is a widely recognized dataset for mathematical reasoning tasks within the RLVR domain. All algorithms were trained for two epochs on this dataset. During the training phase, we set the rollout batch size to 256, with 10 response sequences generated per prompt and a global batch size of 64. For the sampling strategy, both temperature and top-p were set to 1.0 to ensure stochastic diversity. Following standard protocols in mathematical RLVR, the correctness of the final answer served as the sole source of the reward signal.

\paragraph{Code Generation.}

We train Qwen3-4B on the Code-R1~\cite{code-r1} dataset using DAPO as the base algorithm, with and without EDAS. The maximum prompt and response lengths are set to 2,048 and 4,096 tokens, respectively. We generate 10 rollouts per prompt with the same batch-size setting as in mathematical reasoning and train for one epoch. The EDAS hyperparameters are $\alpha = 0.4$, $\beta = 0.2$, $\kappa = 2.0$, consistent with the math experiments.

\section{Baseline Setup}\label{ap:base_setup}

In this section, we describe the baselines used in our experiments.

\paragraph{GRPO and DAPO.}

GRPO serves as a foundational algorithm in the RLVR domain, making it a natural choice for our primary baseline. Building upon this, DAPO~\cite{yu2025dapo} introduces Clip-Higher and Dynamic Sampling techniques, achieving substantially stronger performance than GRPO.

\paragraph{PKPO.}
For PKPO, we set the hyperparameter $k=8$. Given that rewards in the RLVR setting are strictly binary, we adopt the discrete formulation as proposed in the original study. Specifically, let $y_i$ denote the $i$-th generated response and $y^*$ be the ground-truth answer. The reward assignment is defined as:

\begin{equation}
r_i =
\begin{cases}
\dfrac{k}{n} & \text{if } \mathbb{I}(y_i = y^*) = 1 \\[8pt]
\dfrac{k}{n} \cdot \rho(n-1, c, k-1) & \text{if } \mathbb{I}(y_i = y^*) = 0
\end{cases}
\end{equation}

where the term $\rho(n, c, k)$ is defined as the unbiased estimator of the Pass@$k$ rate:

\begin{equation}
    \rho(n, c, k) \equiv 1 - \frac{\binom{n-c}{k}}{\binom{n}{k}} 
\end{equation}

\paragraph{Entropy-Enhanced Advantage.}

To implement the Entropy-Enhanced Advantage baseline, we follow the original configuration by reshaping the advantage values. This approach augments the learning signal with a weighted entropy term to encourage token-level diversity:

\begin{equation}
A_t^{\mathrm{reshape}} = A_t + \min\left( \alpha \cdot \mathcal{H}_t^{\text{detach}},\, \frac{|A_t|}{\kappa} \right),\ \text{where } \alpha > 0 \text{ and } \kappa > 1,
\end{equation}

In this formulation, $\mathcal{H}_t$ represents the token-level entropy at each position $t$:

\begin{equation}
\mathcal{H}_t = -\sum_{v \in \mathcal{V}} \pi_{\theta}(v \mid q, o_{<t}) \log \pi_{\theta}(v \mid q, o_{<t}).
\end{equation}

Following the recommended settings in the literature, we employ a scaling factor $\alpha = 0.4$ and a clipping threshold $\kappa = 2$.


\section{Why Embedding-Based Answer Grouping Cannot Replace Symbolic Equivalence}\label{ap:embedding_failure}

A natural alternative to the symbolic equivalence check used in EDAS's intra-group error partitioning (\mbox{\autoref{sec:partition}}) is to cluster incorrect responses by their \emph{embedding similarity}—representing each generated trajectory as a dense vector and grouping those whose vectors lie close together.
This approach is appealing because it does not require answer extraction and could, in principle, capture semantic differences even when answers are expressed in varied surface forms.
However, we present empirical evidence that embedding similarity is fundamentally unsuitable for distinguishing mathematically correct from incorrect reasoning, and therefore cannot serve as a reliable basis for the diversity signal central to EDAS.

\paragraph{Setup.}
We collected 32 repeated rollouts per problem for all 30 problems of AIME 2025, generated by \texttt{Qwen3-8B} (960 total trajectories).
Final answers were extracted via \texttt{\textbackslash boxed\{\}} pattern matching.
For each problem that elicited at least two numerically distinct answers, we selected up to six representative trajectories (one per distinct answer), embedded each full reasoning chain with \texttt{BAAI/bge-large-en-v1.5}, and computed pairwise cosine similarity.

\paragraph{Deceptive pairs.}
Out of 403 pairs with numerically different answers, \textbf{308} (76.4\%) achieved cosine similarity ${>}0.95$, with an average of 0.978 and a maximum of 1.000. This means that for the vast majority of answer disagreements, an embedding-based clustering method would \emph{merge} the two trajectories into the same group despite their producing different answers.

\paragraph{Deceptive triples.}
To stress-test the failure in the context most directly relevant to EDAS—where we need to distinguish a \emph{correct} trajectory from \emph{two different wrong} trajectories—we identify naturally occurring triples from the model's own rollouts: for each problem, we pair the trajectory reaching the correct answer with two trajectories reaching \emph{distinct} wrong answers, and measure the minimum pairwise similarity across all three pairs.
Table~\ref{tab:embed_triples} summarizes the distribution over 114 such triples.

\begin{table}[h]
\centering
\caption{Minimum pairwise embedding similarity for (correct, wrong$_1$, wrong$_2$) triples on AIME 2025 (Qwen3-8B, 32 rollouts/problem).}
\label{tab:embed_triples}
\begin{tabular}{lc}
\toprule
\textbf{Condition} & \textbf{Count (out of 114)} \\
\midrule
$\min\text{-sim} \ge 0.99$ & 2  \\
$\min\text{-sim} \ge 0.98$ & 19 \\
$\min\text{-sim} \ge 0.95$ & 76 \\
Average $\min\text{-sim}$  & 0.956 \\
Maximum $\min\text{-sim}$  & 0.997 \\
\bottomrule
\end{tabular}
\end{table}

\paragraph{Illustrative example.}
For AIME 2025 Problem 5 (isosceles trapezoid with inscribed circle of radius 3 and area 72; find $r^2+s^2$), 27 out of 32 rollouts arrive at the correct answer 504, while isolated rollouts produce 48 and 424 respectively.
The three pairwise cosine similarities are 0.997, 0.997, and 1.000 (average 0.998).
An embedding-based grouping scheme would classify all three trajectories into the same cluster, completely failing to detect that two of them are erroneous.

\paragraph{Root cause and implications.}
The fundamental reason for this failure is that general-purpose embedding models encode \emph{surface-level textual features}—vocabulary, sentence structure, topic, and style—rather than the \emph{logical validity of a reasoning chain}.
Mathematical reasoning errors often reduce to a single incorrect intermediate step or a one-digit numerical discrepancy, and such differences are absorbed by the high-dimensional embedding space as negligible perturbations.

This finding has a direct implication for EDAS: the equivalence relation used to partition incorrect trajectories into error classes (\autoref{sec:partition}) \emph{must} be grounded in symbolic answer extraction and mathematical canonicalization (via tools such as \texttt{Math-Ruler}), not in embedding proximity.
Using embedding-based grouping as a drop-in replacement would collapse the diversity signal: answers that are numerically far apart (and hence represent genuinely different failure modes) would be counted as the same error type, eliminating the intra-group entropy that EDAS relies on to modulate exploration.


\section{Bottleneck-Breaking Analysis}\label{ap:bottleneck}

A key motivation for EDAS is to help the model escape problems where it consistently fails under the base policy---problems we term \emph{hard} problems, defined as those for which the base model's Average Pass Rate (APR) across 32 repeated rollouts is exactly $0\%$.
These problems represent the hardest frontier of the model's capability: the base policy has \emph{never} produced a correct solution in 32 attempts, so any improvement must come from genuinely new reasoning paths discovered during RL training.
We analyse how many such problems each algorithm ``breaks through'' after training.

\paragraph{Setup.}
We evaluate on four challenging benchmarks: AIME 2025, AIME 2026, AMC 2024, and HMMT Feb 2025, using 32 rollouts per problem.
For each benchmark, we identify the subset of problems where the base model achieves $\text{APR} = 0$, then measure the fraction of these problems for which DAPO and EDAS achieve $\text{APR} > 0$ after training---the \emph{breakthrough rate}.
We conduct this analysis across three model configurations to test generality:
\begin{itemize}
  \item \textbf{Qwen3-8B} (Table~\ref{tab:bottleneck_8b}): the primary setting reported in the main paper.
  \item \textbf{Qwen3-4B} with no-think prompting (Table~\ref{tab:bottleneck_4b_instruct}): a smaller instruct model where the thinking mode is disabled.
  \item \textbf{Qwen3-4B-Base} (Table~\ref{tab:bottleneck_4b_base}): a smaller base (pre-trained only) model without instruction tuning.
\end{itemize}

\paragraph{Results on Qwen3-8B.}
Table~\ref{tab:bottleneck_8b} (in the main paper) summarizes the per-benchmark results on Qwen3-8B.
Across the four benchmarks, 64 problems are classified as hard for the base model.
After training, DAPO breaks through on 22 of these (34.4\%), while EDAS breaks through on 30 (46.9\%), a relative improvement of \textbf{36\%} in breakthrough rate.
Notably, there are \textbf{10 problems} (15.6\% of all hard problems) where EDAS achieves $\text{APR} > 0$ but DAPO remains at $0\%$, demonstrating that EDAS discovers reasoning paths that DAPO never finds.

\paragraph{Results on Qwen3-4B.}
Table~\ref{tab:bottleneck_4b_instruct} shows the same analysis on Qwen3-4B (instruct, no-think mode).
Out of 54 hard problems, DAPO breaks through on 22 (40.7\%) while EDAS breaks through on 24 (44.4\%), with 3 EDAS-only breakthroughs versus only 1 DAPO-only.
The advantage is directionally consistent with Qwen3-8B but more modest (+9.1\% relative improvement), likely because the instruction-tuned model already possesses some exploration diversity inherited from supervised fine-tuning, leaving less room for EDAS's diversity mechanism to add value.

\begin{table}[h]
\centering
\caption{Bottleneck-breaking analysis on Qwen3-4B (Instruct, no-think). Hard = base model $\text{APR}=0$ over 32 rollouts.}
\label{tab:bottleneck_4b_instruct}
\begin{tabular}{lcccc}
\toprule
\textbf{Benchmark} & \textbf{Hard} & \textbf{DAPO Break} & \textbf{EDAS Break} & \textbf{EDAS-only} \\
\midrule
AIME 2025    & 10 &  6 & \textbf{ 6} & 1 \\
AIME 2026    & 13 &  6 & \textbf{ 6} & 0 \\
AMC 2024     & 12 &  5 & \textbf{ 6} & 1 \\
HMMT Feb 25  & 19 &  5 & \textbf{ 6} & 1 \\
\midrule
\textbf{Total} & \textbf{54} & \textbf{22} (40.7\%) & \textbf{24} (44.4\%) & \textbf{3} \\
\bottomrule
\end{tabular}
\end{table}

\paragraph{Results on Qwen3-4B-Base.}
Table~\ref{tab:bottleneck_4b_base} presents the results on Qwen3-4B-Base, the most revealing setting.
Of 90 hard problems, DAPO breaks through on only 18 (20.0\%), while EDAS breaks through on 31 (34.4\%)---a relative improvement of \textbf{72.2\%}, nearly double the gain observed on Qwen3-8B.
The EDAS-only count reaches \textbf{14 problems} (15.6\% of all hard problems), compared to merely \textbf{1} DAPO-only breakthrough.
The effect is especially striking on AIME 2025, where DAPO achieves zero breakthroughs while EDAS breaks through on 4 problems, and on HMMT Feb 2025 (DAPO: 1, EDAS: 6).

\begin{table}[h]
\centering
\caption{Bottleneck-breaking analysis on Qwen3-4B-Base. Hard = base model $\text{APR}=0$ over 32 rollouts.}
\label{tab:bottleneck_4b_base}
\begin{tabular}{lcccc}
\toprule
\textbf{Benchmark} & \textbf{Hard} & \textbf{DAPO Break} & \textbf{EDAS Break} & \textbf{EDAS-only} \\
\midrule
AIME 2025    & 18 &  0 & \textbf{ 4} & 4 \\
AIME 2026    & 24 &  4 & \textbf{ 7} & 4 \\
AMC 2024     & 25 & 13 & \textbf{14} & 1 \\
HMMT Feb 25  & 23 &  1 & \textbf{ 6} & 5 \\
\midrule
\textbf{Total} & \textbf{90} & \textbf{18} (20.0\%) & \textbf{31} (34.4\%) & \textbf{14} \\
\bottomrule
\end{tabular}
\end{table}

\paragraph{Discussion.}
The superiority of EDAS on hard problems is consistent across all three model scales and configurations, confirming that the effect is systematic rather than model-specific.
Because EDAS actively penalizes mode collapse and promotes rare error types, the policy is continuously nudged away from homogeneous failure modes.
When the base model fails all 32 attempts, it typically does so by converging to the same incorrect reasoning chain; DAPO inherits this mode-collapsed behavior and struggles to escape it.
EDAS's diversity-encouraging mechanism diversifies the exploration landscape even for the hardest problems, increasing the probability that at least one rollout discovers a correct reasoning path---thereby generating a positive reward signal that DAPO never receives.

Comparing across model configurations reveals an informative gradient.
The advantage is largest on Qwen3-4B-Base (+72.2\% relative), moderate on Qwen3-8B (+36\%), and smallest on Qwen3-4B Instruct (+9.1\%).
This ordering aligns with the degree of mode collapse in the initial policy: base models without instruction tuning are most prone to homogeneous failure modes, and therefore benefit the most from EDAS's diversity-encouraging objective.
Instruction-tuned models already exhibit some exploration diversity from supervised fine-tuning, leaving less headroom for improvement.
The consistent EDAS-only breakthroughs across all settings and benchmarks---10 on 8B, 14 on 4B-Base, and 3 on 4B Instruct---confirm that EDAS systematically discovers reasoning paths that DAPO never finds, regardless of model scale.


\section{Internal Dynamics of EDAS's Advantage Adjustment}\label{ap:internal}

We log EDAS's internal statistics throughout training (2 epochs on DAPO-Math-17K) and compare them against the same group-level metrics passively recorded from DAPO on the identical training setup (Qwen3-8B, DAPO-Math-17K, 2 epochs, batch size 256).
This provides a controlled comparison of how the two algorithms shape the error landscape during RLVR.
Note that for DAPO, these metrics are observed but not acted upon; EDAS actively modulates advantages based on them.

\subsection{Metric Definition}

For each training step, we count the number of prompt groups exhibiting healthy error diversity ($K > 1$, i.e., at least two distinct wrong answer types coexist among the incorrect trajectories).
This metric directly measures the breadth of the model's exploration: more diverse groups mean more varied failure modes, providing richer gradient signal for discovering correct reasoning paths.
We compute this identically for both EDAS and DAPO: for each group, we extract incorrect answers from the rollouts, canonicalise them, and count the number of distinct error types $K$.

\subsection{EDAS Preserves Significantly More Diverse Groups}

As shown in Figure~\ref{fig:internal}, EDAS consistently maintains far more diverse groups than DAPO throughout training.
Averaged across the entire training run (2 epochs), EDAS preserves a mean of 144.1 diverse groups per step compared to DAPO's 89.5---a \textbf{1.61$\times$} ratio overall.
The advantage grows over training: during the first epoch the ratio averages 1.45$\times$, while during the second epoch it widens to \textbf{1.98$\times$}.

This widening gap reflects a fundamental dynamic: as standard RLVR training concentrates the policy on higher-reward patterns, error diversity naturally erodes.
DAPO's diverse groups decline from 182 to as low as 54---a 70.3\% collapse.
EDAS counteracts this erosion: its diverse groups decline more gently (from 189 to a minimum of 112) and partially recover in late training, ending at 146.
The recovery is driven by EDAS's Branch B penalty, which forces mode-collapsed groups ($K{=}1$) to diversify, replenishing the pool of diverse groups even as training progresses.

\begin{figure}[tp]
    \centering
    \includegraphics[width=0.55\linewidth]{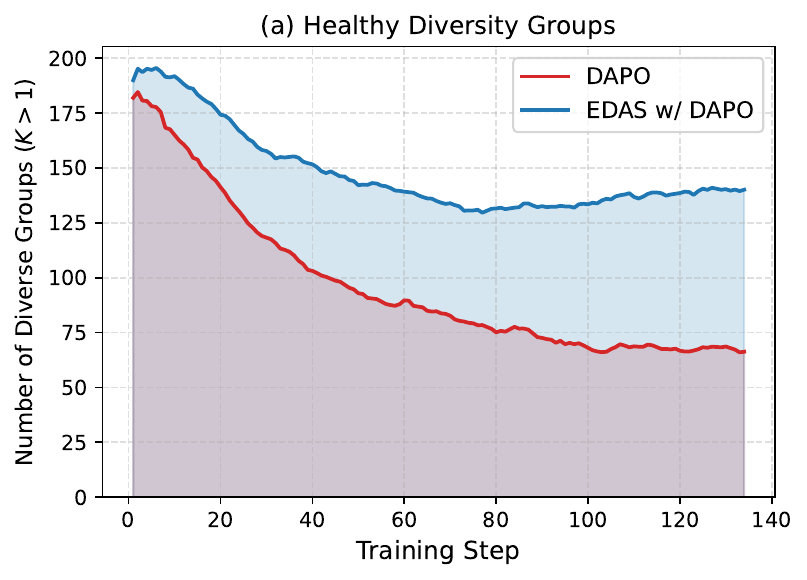}
    \caption{Number of groups with healthy error diversity ($K > 1$) across 2 epochs of training. EDAS (blue) maintains on average 1.71$\times$ more diverse groups than DAPO (red). The gap widens from 1.45$\times$ in the first epoch to 1.98$\times$ in the second epoch, demonstrating that EDAS's diversity preservation strengthens precisely as the remaining problems grow harder.}
    \label{fig:internal}
\end{figure}

\subsection{Implications}

This result provides direct mechanistic evidence for EDAS's effectiveness:

\begin{itemize}
    \item \textbf{Diversity preservation is effective.} Despite both algorithms starting from the same initial diversity (189 vs.\ 182 at step 1), EDAS's Branch B/C mechanism prevents the severe diversity collapse that DAPO experiences (70.3\% decline vs.\ 40.7\%).
    \item \textbf{The effect strengthens when it matters most.} The EDAS/DAPO ratio widens from 1.45$\times$ to 1.98$\times$ from the first to the second epoch. In the second epoch, the remaining unsolved problems are the hardest, and maintaining diverse error types on these problems is critical for the model to discover new reasoning paths---exactly the mechanism behind the bottleneck-breaking improvements reported in \autoref{ap:bottleneck}.
    \item \textbf{EDAS actively reverses diversity loss.} Rather than merely slowing collapse, EDAS's diverse group count partially \emph{recovers} in the second epoch (from a minimum of 112 back to 146), indicating that Branch B successfully forces collapsed groups to re-diversify.
\end{itemize}

\section{Limitations}\label{ap:limitations}

While EDAS demonstrates consistent improvements across multiple models and domains, we identify several limitations that contextualize the current results and suggest directions for future work.

\paragraph{Reliance on Observable Output Features for Error Partitioning.}
EDAS defines error equivalence through observable output features---final numerical answers for mathematical reasoning and exception types for code generation. This design, while effective and computationally efficient, cannot distinguish different reasoning paths that happen to produce the same observable output. For instance, two trajectories may arrive at the same incorrect numerical answer through entirely different logical errors, yet EDAS would classify them as belonging to the same equivalence class. Incorporating process-level diversity signals, e.g., based on intermediate reasoning steps or execution traces, could yield a finer-grained error taxonomy but would introduce significant computational overhead and design complexity.

\paragraph{Domain-Specific Engineering of the Equivalence Function.}
Although EDAS's downstream computation is fully domain-agnostic, the equivalence function $\mathcal{E}$ must be manually designed for each new domain. For mathematical reasoning and code generation, natural choices exist (canonicalized answers and exception types, respectively). However, extending EDAS to domains such as open-ended text generation, multi-step tool use, or formal theorem proving requires identifying appropriate discrete error labels---a non-trivial design decision that may limit out-of-the-box applicability.

\paragraph{Fixed Hyperparameters Across Training.}
The current formulation uses fixed hyperparameters ($\alpha$, $\beta$, $\kappa$) throughout training. As the policy evolves, the optimal balance between diversity encouragement and collapse penalty may shift: early training may benefit from stronger exploration incentives, while later stages may require more conservative adjustments to avoid destabilizing a near-converged policy. Adaptive scheduling of these hyperparameters remains unexplored.

\paragraph{Sensitivity to Group Size and Sampling Budget.}
EDAS's information-theoretic quantities (entropy, self-information) are estimated from finite samples within each rollout group. With the group size fixed at $N=10$ in our experiments, the entropy estimate may be noisy for problems that admit many distinct error types, potentially underestimating the true diversity. Larger group sizes would improve estimation quality but proportionally increase computational cost. Exploring efficient training architectures~\cite{xu2026grouter} to reduce the overhead of larger rollout groups is a promising direction. The interaction between group size, diversity estimation accuracy, and downstream performance has not been systematically characterized.

\paragraph{Evaluation Scope.}
Our experiments focus on mathematical reasoning (competition-level problems) and code generation. While these represent two important verifiable domains, the generalizability of EDAS to other settings---such as formal theorem proving, scientific reasoning, multimodal visual reasoning~\cite{dong2025interleaved}, or multi-hop question answering with verifiable answers---remains empirically unvalidated. The consistently positive results across our tested domains are encouraging, but do not guarantee transferability to domains with fundamentally different error structures.
\end{document}